\documentclass[10pt,twocolumn,letterpaper]{article}

\usepackage[pagenumbers]{cvpr} 

\usepackage{graphicx}
\usepackage{amsmath}
\usepackage{amssymb}
\usepackage{booktabs}

\usepackage{bm}
\usepackage{amsmath,amssymb,amsthm}
\usepackage{multirow,array}
\usepackage{color}
\usepackage{algorithmicx,algorithm, algpseudocode}
\usepackage{lipsum}
\usepackage{booktabs}
\usepackage{graphicx}
\usepackage{comment}
\usepackage{arydshln}
\usepackage{bbding}
\usepackage{listings}
\newtheorem{proposition}{Proposition}

\usepackage{diagbox}

\definecolor{codegreen}{rgb}{0,0.6,0}
\definecolor{codegray}{rgb}{0.5,0.5,0.5}
\definecolor{codepurple}{rgb}{0.58,0,0.82}
\definecolor{backcolour}{rgb}{0.95,0.95,0.92}

\lstdefinestyle{mystyle}{
    commentstyle=\color{codegreen},
    keywordstyle=\color{magenta},
    numberstyle=\tiny\color{codegray},
    stringstyle=\color{codepurple},
    basicstyle=\ttfamily\footnotesize,
    breakatwhitespace=false,         
    breaklines=true,                 
    captionpos=b,                    
    keepspaces=true,                 
    numbers=left,                    
    numbersep=5pt,                  
    showspaces=false,                
    showstringspaces=false,
    showtabs=false,                  
    tabsize=2
}
\usepackage[pagebackref,breaklinks,colorlinks]{hyperref}

\usepackage[capitalize]{cleveref}
\crefname{section}{Sec.}{Secs.}
\Crefname{section}{Section}{Sections}
\Crefname{table}{Table}{Tables}
\crefname{table}{Tab.}{Tabs.}

\def\cvprPaperID{4121} 
\def\confName{CVPR}
\def\confYear{2023}

\begin{document}

\title{A Unified Framework for Contrastive Learning\\ from a Perspective of Affinity Matrix}

\author{Wenbin Li\textsuperscript{$1$}, Meihao Kong\textsuperscript{$1$}, Xuesong Yang\textsuperscript{$1$}, Lei Wang\textsuperscript{$2$}, Jing Huo\textsuperscript{$1$},Yang Gao\textsuperscript{$1$}, Jiebo Luo\textsuperscript{$3$} \\
\textsuperscript{$1$}State Key Laboratory for Novel Software Technology, Nanjing University, China \\
\textsuperscript{$2$}University of Wollongong, Australia,
\textsuperscript{$3$}University of Rochester, USA\\
}

\maketitle

\begin{abstract}
   In recent years, a variety of contrastive learning based unsupervised visual representation learning methods have been designed and achieved great success in many visual tasks. Generally, these methods can be roughly classified into four categories: (1) standard contrastive methods with an InfoNCE like loss, such as MoCo and SimCLR; (2) non-contrastive methods with only positive pairs, such as BYOL and SimSiam; (3) whitening regularization based methods, such as W-MSE and VICReg; and (4) consistency regularization based methods, such as CO2. In this study, we present a new unified contrastive learning representation framework (\textit{named UniCLR}) suitable for all the above four kinds of methods from a novel perspective of basic affinity matrix. Moreover, three variants, \emph{i.e.,} \textit{SimAffinity}, \textit{SimWhitening} and \textit{SimTrace}, are presented based on UniCLR. In addition, \textit{a simple symmetric loss}, as a new consistency regularization term, is proposed based on this framework. By symmetrizing the affinity matrix, we can effectively accelerate the convergence of the training process. Extensive experiments have been conducted to show that (1) the proposed UniCLR framework can achieve superior results on par with and even be better than the state of the art, (2) the proposed symmetric loss can significantly accelerate the convergence of models, and (3) SimTrace can avoid the mode collapse problem by maximizing the trace of a whitened affinity matrix without relying on asymmetry designs or stop-gradients.
\end{abstract}

\section{Introduction}
With the goal of learning effective visual representations without human supervision, self-supervised learning (SSL), especially the contrastive learning based SSL, has attracted increasing attention in recent years~\cite{wu2018unsupervised,he2020momentum,chen2020simple,chen2021exploring,grill2020bootstrap,tian2020makes,caron2020unsupervised,zbontar2021barlow,PixPro_CVPR2021,chen2021empirical}. A variety of SSL methods have been proposed and significantly advanced the progress of unsupervised visual representation learning. Typically, these methods could be roughly divided into four categories: (1) standard contrastive SSL, (2) non-contrastive SSL, (3) whitening regularization based SSL, and (4) consistency regularization based SSL, where the latter two can be easily integrated with the former two.

Specifically, standard contrastive SSL methods, such as MoCo~\cite{he2020momentum,chen2020improved} and SimCLR~\cite{chen2020simple}, normally maximize the agreement between two different augmentations of the same image, \emph{i.e.,} a positive pair, by simultaneously contrasting a lot of different images, \emph{i.e.,} many negative pairs. Generally, such a kind of method requires either a large memory bank or a large batch size to achieve excellent performance. In contrast, the non-contrastive SSL methods, such as SwAV~\cite{caron2020unsupervised}, BYOL~\cite{grill2020bootstrap} and SimSiam~\cite{chen2021exploring}, try to discard the massive negative pairs and instead only maximize the similarity between positive pairs. To avoid the collapse problem for not using negative pairs, different asymmetric designs are commonly used, such as asymmetric embedding network (\emph{e.g.,} additional predictor network)~\cite{grill2020bootstrap,chen2021exploring}, momentum encoder~\cite{grill2020bootstrap}, non-differentiable operators (\emph{e.g.,} clustering with Sinkhorm-Knopp algorithm~\cite{SK_NeurIPS2013})~\cite{caron2020unsupervised}, and stop-gradients~\cite{grill2020bootstrap,chen2021exploring}. Similarly, the whitening regularization based SSL methods, such as W-MSE~\cite{W_MSE_ICML2021}, Barlow Twins~\cite{zbontar2021barlow} and VICReg~\cite{VICReg_ICLR2022}, employ whitening with a decorrelation operation to avoid the collapse representations in the non-contrastive SSL without requiring asymmetric designs. Different from the above three kinds of methods, consistency regularization based SSL methods, such as CO2~\cite{CO2_ICLR2021}, use a semi-supervised consistency regularization to alleviate the problem of false negatives caused by the use of the instance discrimination pretext task.

In this paper, we propose a \textit{simple UNIfied Contrastive Learning Representation framework, called UniCLR}, to unify the above four kinds of SSL methods to the maximum extent. The motivation is that although these methods are different in terms of their motivations or the employed techniques, they are essentially closely related to each other. From a new perspective of affinity matrix, we notice that these methods can be unified into a single and much simpler framework. That is to say, most of the existing SSL methods would be a special case or variant of this framework. Based on the UniCLR framework, we first present a simple yet effective baseline, \textit{SimAffinity}, by directly optimizing the affinity matrix of two different views of the same mini-batch with the cross-entropy loss. After that, we propose a simple whitening based baseline based on SimAffinity, named \textit{SimWhitening}, by incorporating the whitening operation. Moreover, based on SimWhitening, we further design a novel non-contrastive SSL method, named \textit{SimTrace}, by directly optimizing the trace of a whitened affinity matrix. In particular, SimTrace does not rely on asymmetry designs or regularization terms to avoid the collapse problem. In addition, we propose \textit{a new symmetric consistency regularization loss} by symmetrizing the affinity matrix, which can significantly accelerate the convergence of unsupervised SSL models.

To verify the effectiveness and scalability of this new framework, we conduct extensive experiments on ImageNet-1K~\cite{deng2009imagenet} and three small benchmark datasets, including multiple ablation studies to investigate the major components of UniCLR. Importantly, as a unified framework, our UniCLR can be flexibly integrated with some specific tricks, such as momentum encoder~\cite{grill2020bootstrap} and multi-crop~\cite{caron2020unsupervised}, to achieve new state-of-the-art results on ImageNet-1K. Also, all the variants of UniCLR can obtain competitive or even much better results than the corresponding SSL methods. Specifically, under the linear evaluation protocol, SimAffinity achieves $73.8\%$ top-1 accuracy on ImageNet-1K with 200-epoch training, which is a $3.2\%$ absolute improvement over BYOL~\cite{grill2020bootstrap}.

In summary, our main contributions are as follows:
\begin{itemize}
    \item We present \textit{UniCLR}, a unified contrastive learning framework, to represent multiple different kinds of SSL methods. In addition, three new variants derived from UniCLR, especially \textit{SimAffinity}, can achieve superior results over the state of the art.
    \item We propose \textit{SimTrace}, a novel non-contrastive SSL method built on UniCLR, by optimizing the affinity matrix through a trace operation. Theoretical analysis is provided to show that the affinity matrix and covariance matrix can be simultaneously constrained with a single loss term.
   \item We propose \textit{a new symmetric consistency regularization loss} on the affinity matrix, which can significantly accelerate the convergence of models and boost the performance by alleviating the issue of false negatives.
   \item We comprehensively demonstrate the effectiveness of the proposed framework and the newly designed variants on multiple benchmarks, and provide new strong baselines for SSL.
\end{itemize}

\section{Related Work}
\textbf{Contrastive Learning.} 
Contrastive learning based SSL methods aim to learn a feature embedding space by simultaneously maximizing the agreement between positive examples and the disagreement with negative examples. That is to say, both positive and negative examples will be involved in the unsupervised pre-training. A variety of methods~\cite{hjelm2019learning,oord2018representation,henaff2020data,chen2020simple,he2020momentum,chen2020improved,hu2021adco} have been proposed along this paradigm, such as NPID~\cite{wu2018unsupervised}, CMC~\cite{tian2019contrastive}, CPC v2~\cite{henaff2020data}, InfoMin~\cite{tian2020makes}, MoCo~\cite{he2020momentum,chen2020improved}, SimCLR~\cite{chen2020simple}, PCL~\cite{PCL_ICLR2021}, AdCo~\cite{hu2021adco} and so on. In generally, an InfoNCE like loss is used and optimized in this type of methods, so as to employ a large memory bank or a large batch size to obtain massive negative examples for excellent performance.

\textbf{Non-Contrastive Learning.}
Different from the contrastive SSL methods, non-contrastive learning based SSL methods try to completely discard the large number of negative examples for reducing the computation overhead. Also, to avoid trivial solutions caused by not using negatives as anchors, some asymmetric structures or operations are normally designed, such as asymmetric embedding network~\cite{grill2020bootstrap,chen2021exploring}, momentum encoder~\cite{grill2020bootstrap}, non-differentiable operators~\cite{caron2020unsupervised}, and stop-gradients~\cite{grill2020bootstrap,chen2021exploring}. For example, BYOL~\cite{grill2020bootstrap} predicts the representation of one view via another view based on bootstrapping by using an online network and a momentum network, where both asymmetric embedding network and stop-gradients are used. After that, SimSiam~\cite{chen2021exploring} discards the momentum encoder in BYOL and shows that only using an asymmetric embedding network and stop-gradients can achieve excellent performance with a small batch size. Recently, some latest works attempt to introduce vision transformers instead of convolutional neural networks into SSL, such as DINO~\cite{caron2021emerging} and EsViT~\cite{li2021efficient}.

\textbf{Whitening Regularization.}
Whitening transformation is a classical pre-processing step to transform random variables to orthogonality, which has been frequently introduced into SSL in recent works. Specifically, W-MSE~\cite{W_MSE_ICML2021} scatters all the sample representations into a spherical distribution using a whitening transform and penalizes the positive pairs which are far from each other to avoid the collapse problem. Barlow Twins~\cite{zbontar2021barlow} proposes to makes the cross-correlation matrix as close to the identity matrix as possible by reducing redundancy between their components to learn unsupervised representations. In addition, VICReg~\cite{VICReg_ICLR2022} explicitly makes the embedding vectors of samples within a batch to be different via a new variance term based on the former work of Barlow Twins. DBN~\cite{DBN} proposes to standardize the variance and the covariance matrix to address the complete collapse and dimensional collapse issues, respectively. Recently, Zero-CL~\cite{Zero-CL} fully combines both instance-wise whitening and feature-wise whitening into the contrastive learning to prevent trivial solutions.

\begin{figure*}[!tp]
\centering
\includegraphics[width=0.65\textwidth]{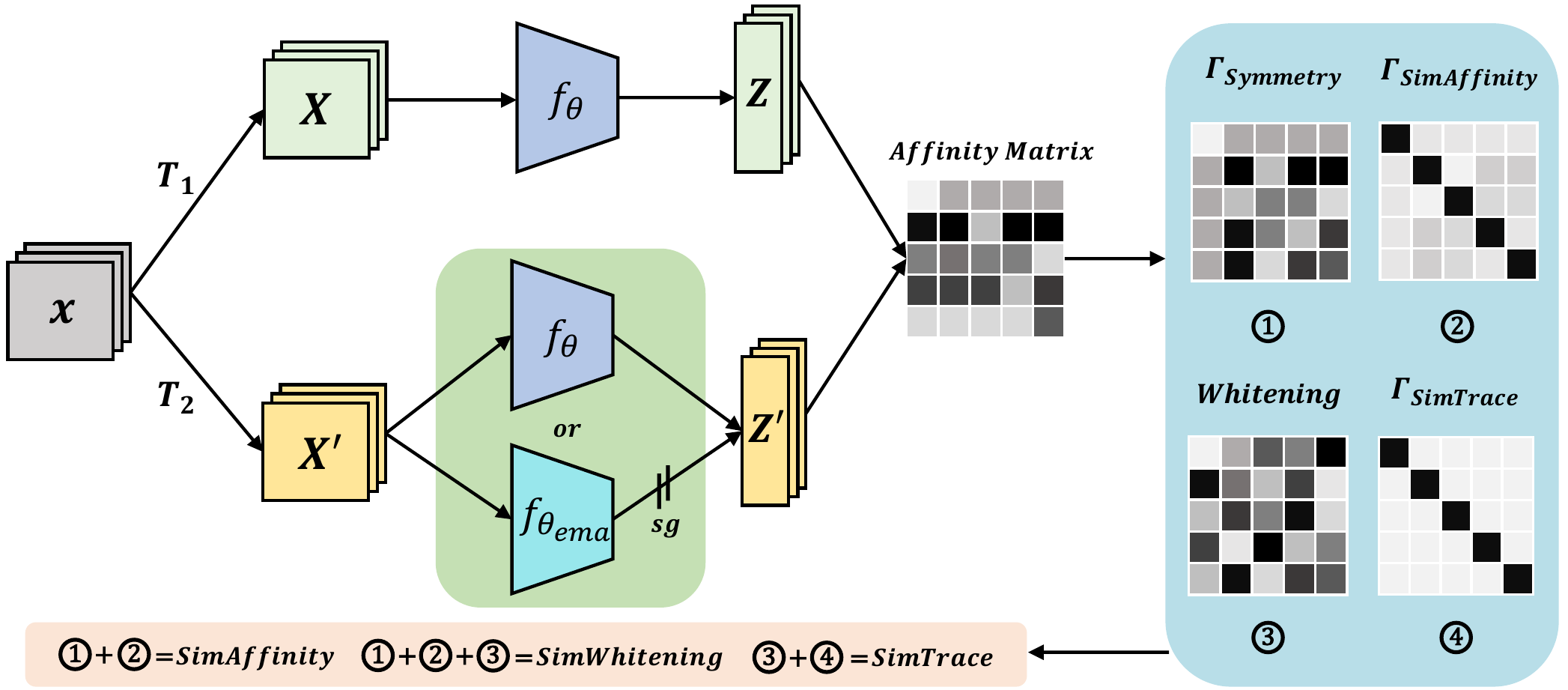}
\vspace{-1mm}
\caption{Overall architecture of the proposed UniCLR framework. Given a mini-batch of raw images $\bm{x}$, two sets of views $\bm{X}$ and $\bm{X'}$ are obtained via data augmentation and then encoded into representations $\bm{Z}$ and $\bm{Z'}$ through the weights-shared Siamese networks $f_{\theta}$ or one of which $f_{\theta_{ema}}$ is updated with the exponential moving average (EMA) of $f_{\theta}$ with a stop-gradient (\emph{i.e.,} sg) operation. The affinity matrix between $\bm{Z}$ and $\bm{Z'}$ is finally calculated, forming the basis of the variants of UniCLR, \emph{i.e.,} SimAffinity, SimWhitening and SimTrace.}
\vspace{-3mm}
\label{Framwork}
\end{figure*}

\textbf{Consistency Regularization.}
As a core idea of many state-of-the-art semi-supervised learning methods~\cite{berthelot2019mixmatch,sohn2020fixmatch}, consistency regularization is also introduced into the field of SSL to further boost the performance of unsupervised models, recently. For example, CO2~\cite{CO2_ICLR2021} exploits the consistency assumption that a robust model should generate similar representation vectors on perturbed versions of the same image to encourage the consistent similarities between image crops. Similarly, TWIST~\cite{TWIST} enforces the class probability distributions of two augmented views of the same image to be consistent while regularizing the distributions to make them sharp and diverse. Both of methods demonstrate the effectiveness of the consistency regularization term through extensive experiments.

Our UniCLR framework is closely related to all the above four types of methods. The key difference is that UniCLR attempts to unify the above four types of methods into a simple framework from a new affinity matrix perspective. Moreover, three variants, \emph{i.e.,} SimAffinity, SimWhitening and SimTrace, are designed based on UniCLR and show remarkable performance.

\section{Methods}
\subsection{An Affinity Matrix Perspective}
\label{section_3_1}
Following the literature~\cite{he2020momentum,chen2020simple,chen2021exploring,grill2020bootstrap,tian2020makes}, let $\mathcal{X}=\{X_1, \cdots, X_N\}$ and $\mathcal{X'}=\{X'_1, \cdots, X'_N\}$ denote two different views of the same mini-batch of images via data augmentation. As shown in Fig.~\ref{Framwork}, through an embedding network $f_\theta(\cdot)$, which is composed of a backbone and a multilayer perceptron (MLP) based projector, we can obtain the feature matrices of the two batches of images, \emph{i.e.,} $Z=f_\theta(\mathcal{X})\in\mathbb{R}^{D\times N}$ and $Z'=f_\theta(\mathcal{X'})\in\mathbb{R}^{D\times N}$, where $D$ is the dimension of features, $N$ is the number of raw images in each mini-batch and $\theta$ denotes the parameters of the embedding network. Also, if we want to introduce the momentum encoder, one view of images can be represented as $Z'=f_{\theta_{ema}}(\mathcal{X'})\in\mathbb{R}^{D\times N}$, where the parameters $\theta_{ema}$ are updated with the exponential moving average (EMA) of the parameters of $f_\theta(\cdot)$ as well as a stop-gradient operation.

\textbf{SimAffinity: A simple affinity matrix based baseline.} As one representative framework, SimCLR~\cite{chen2020simple} normally constructs all the possible negative pairs, \emph{i.e.,} $2(N-1)$ negative pairs, for each positive pair. However, incorporating such many negative pairs will significantly increase the computational overhead. Different from SimCLR, in this work, we consider the contrastive learning task from a new perspective of affinity matrix, \emph{i.e.,} only $N-1$ negatives are considered for each positive. To be specific, an affinity matrix based contrastive learning loss can be formulated as: 
\begin{equation}\label{fun1}\small
\begin{split}
 \mathnormal{\Gamma}_\text{SimAffinity}=\mathcal{L}^\text{CE}(Z^\top Z',Y)\,,
\end{split}
\end{equation}
where $Z^\top Z'\in \mathbb{R}^{N\times N}$ is the affinity matrix between $Z$ and $Z'$, both of which are generally $\ell_2$-normalized along the feature dimension, $Y\in \mathbb{R}^{N\times N}$ denotes the pseudo labels whose each row is a one-hot vector, and $\mathcal{L}^\text{CE}$ is the cross-entropy loss function.

We highlight that SimAffinity can be a representative to represent the InfoNCE loss based standard contrastive SSL methods, such as MoCo and SimCLR (\textit{further explanation can be seen in the Appendix~\ref{appendix_C}}). More importantly, in Section~\ref{section_experiments}, we will show that SimAffinity can significantly overcome MoCo and SimCLR with a simpler formulation by incorporating a new tailored symmetric loss.

\textbf{SimWhitening: A simple whitening-based baseline.} Typically, if we combine $Z$ with $Z'$ together, \emph{i.e.,} $\hat{Z}=\{Z, Z'\}\in\mathbb{R}^{D\times 2N}$, we can obtain the mean vector $\bm{\mu}=mean(\hat{Z})\in\mathbb{R}^D$ and covariance matrix $\Sigma=(\hat{Z}-\bm{\mu})(\hat{Z}-\bm{\mu})^\top\in\mathbb{R}^{D\times D}$. Suppose we are able to obtain the whitening matrix $W=\Sigma^{-\frac{1}{2}}\in\mathbb{R}^{D\times D}$ by eigen-decomposition, we have the whitened features $\bar{Z}=W\cdot Z=\Sigma^{-\frac{1}{2}}Z$ and $\bar{Z'}=W\cdot Z'=\Sigma^{-\frac{1}{2}}Z'$, where both $Z$ and $Z'$ have been subtracted with the mean $\bm{\mu}$. In this sense, we calculate the affinity matrix again and obtain the following loss function:
\begin{equation}\label{fun2}\small
\begin{split}
 \mathnormal{\Gamma}_\text{SimWhitening}&=\mathcal{L}^\text{CE}(\bar{Z}^\top\bar{Z'},Y)\\
                    &=\mathcal{L}^\text{CE}\big((\Sigma^{-\frac{1}{2}}Z)^\top\Sigma^{-\frac{1}{2}}Z',Y\big) \\
                    &=\mathcal{L}^\text{CE}\big(Z^\top\Sigma^{-\frac{1}{2}}\Sigma^{-\frac{1}{2}}Z',Y\big) \\
                    &=\mathcal{L}^\text{CE}\big(Z^\top\Sigma^{-1}Z',Y\big)\,.
\end{split}
\end{equation}
As seen, the affinity matrix after whitening is essentially a bilinear similarity between $Z$ and $Z'$ incorporating with the inverse of covariance matrix.

Different from the existing whitening based SSL methods that mainly consider whitening under the non-contrastive setting, here we show that the whitening operation can be inserted into the contrastive setting with a single loss function.

\textbf{SimTrace: A simple trace-based baseline.} Based on the above Eq.~(\ref{fun2}), we are able to obtain a non-contrastive method \textit{SimTrace} by applying a trace operation on the whitened affinity matrix. That is to say, the loss function of SimTrace can be formulated as follows: 
\begin{equation}\label{fun3}\small
\begin{split}
 \mathnormal{\Gamma}_\text{SimTrace}=-\text{Trace}\big(Z^\top\Sigma^{-1}Z')\,,
\end{split}
\end{equation}
where $\text{Trace}(\cdot)$ denotes the trace of a square matrix. That is to say, by utilizing a simple trace operation, the sum of elements on the main diagonal of the whitened affinity matrix is maximized. In another word, the sum of similarities between positive pairs is taken as the objective to maximize, which is in line with the core idea of non-contrastive SSL methods. In addition, according to the basic properties of trace (\textit{the proof is in the Appendix~\ref{appendix_B}}), we have 
\begin{equation}\label{fun4}\small
\begin{split}
 \mathnormal{\Gamma}_\text{SimTrace}=-\text{Trace}\big(\Sigma^{-1}Z'Z^\top)\,,
\end{split}
\end{equation}
where $Z'Z^\top$ is in fact the cross-covariance matrix between $Z'$ and $Z$, and $\Sigma^{-1}$ here can be seen as a linear transformation matrix. In this sense, the covariance matrix is essentially implicitly considered as an objective to be optimized with the same single trace loss function.

We highlight that SimTrace can typically represent the whitening based non-contrastive SSL methods, such as W-MSE~\cite{W_MSE_ICML2021}, Barlow Twins~\cite{zbontar2021barlow} and VICReg~\cite{VICReg_ICLR2022}. Notably, like these whitening based SSL methods, SimTrace can naturally avoid the collapse problem by the whitening operation with a weights-shared Siamese network, rather than requiring either asymmetric network designs or gradient stopping for achieving this purpose. On the contrary, these asymmetric designs (like training tricks), \eg, asymmetric predictor~\cite{grill2020bootstrap,chen2021exploring}, can be flexibly integrated with SimTrace to further enhance its performance. More importantly, different from the existing methods, \textit{SimTrace does not require explicitly calculating the matrix square rooting $\Sigma^{-\frac{1}{2}}$ or explicitly performing the whitening operation} (\emph{i.e.,} do not need to obtain the whitening matrix $\Sigma^{-\frac{1}{2}}$ by eigen-decomposition or Cholesky decomposition). We even do not use $\ell_2$ normalization for SimTrace.

\textbf{Symmetry: A simple symmetric loss.} One limitation of the instance discrimination pretext task is that two different images will be crudely taken as two different pseudo classes even though they may belong to the same semantic class. In other words, the false negative problem clearly exists and has not yet been effectively addressed in contrastive learning based SSL. To address this problem, CO2~\cite{CO2_ICLR2021} introduces a Kullback-Leibler (KL) divergence based consistency regularization term to encourage the consistency between the similarities of the positive pair with respect to all the negatives. Different from CO2, benefitting from using affinity matrix as the objective, we propose a much simpler symmetric consistency regularization loss as below:  
\begin{equation}\label{fun5}\small
\begin{split}
 \mathnormal{\Gamma}_\text{Symmetry}= \|Z^\top Z'-Z'^\top Z\| \ \ \text{or}\ \ \|Z^\top\Sigma^{-1}Z'-Z'^\top\Sigma^{-1}Z\|\,,
\end{split}
\end{equation}
where $\|\cdot\|$ denotes the Frobenius norm of matrix. That is to say, we just need to ensure the symmetry of the obtained affinity matrix, which can somewhat achieve the effect of alleviating the false negative problem in SSL. This is because, in general, we will only use the pseudo one-hot ground truth to maximize the similarity between the positive pair while minimizing the similarities between all negative pairs, including the false negative pairs. In contrast, using such a symmetry loss, the relative similarities (which are normally large) of the false negative pairs will be largely maintained to some extent (\textit{further explanation is shown in the Appendix~\ref{appendix_E}}).

Another difference with the work of CO2~\cite{CO2_ICLR2021} is that, the two views of negatives are shuffled together in CO2, while our symmetric loss encourages the consistency between one view and another view by calculating the cross positive-negative similarity. In other words, the pair of positive and negative comes from different views, which can benefit the final performance owing to the diversity. More importantly, another advantage of the proposed symmetric loss is that it can effectively avoid the saturated gradients suffered of using KL divergence.

\subsection{UniCLR: A UNIfied Contrastive Learning Representation Framework}
According to Section~\ref{section_3_1}, we are able to design and unify four different kinds of SSL methods into a single framework from a perspective of affinity matrix. Therefore, we formulate the overall framework as follows:
\begin{equation}\label{fun6}\small
\begin{split}
 \mathnormal{\Gamma}_\text{UniCLR}=\mathcal{L}\big(Z^\top\textcolor{blue}{\Sigma^{-1}}Z'/\tau\big) +\gamma\cdot \|Z^\top\textcolor{blue}{\Sigma^{-1}}Z'-Z'^\top\textcolor{blue}{\Sigma^{-1}}Z\|\,,
\end{split}
\end{equation}
where $\tau$ is an optional temperature parameter, $\textcolor{blue}{\Sigma^{-1}}$ is also optional according to whether whitening is used or not, and $\gamma$ is a hyper-parameter controlling the importance of the symmetry regularization term. Specifically, for the variant \textit{SimTrace}, only the first term of Eq.~(\ref{fun6}) is used and $\mathcal{L}(\cdot)$ is replaced with a negative trace operation (see Eq.~(\ref{fun3})). As for the variants \textit{SimAffinity} and \textit{SimWhitening}, both of the two terms of Eq.~(\ref{fun6}) are normally retained and the standard cross-entropy loss is used as $\mathcal{L}(\cdot)$. In addition, we observe that for the contrastive learning setting, \emph{i.e.,} negatives are involved in the optimization process, $\ell_2$ normalization (\emph{i.e.,} cosine similarity) benefits the final performance. Therefore, for \textit{SimWhitening}, $\ell_2$ normalization is performed on each view, which means that Cholesky decomposition is first performed before calculating the affinity matrix.

Note that UniCLR mainly focuses on designing a unified loss function for contrastive-learning based SSL methods, and do not concern about the architecture designs of the network (\eg, asymmetric predictor~\cite{grill2020bootstrap,chen2021exploring}, momentum encoder~\cite{grill2020bootstrap}, and stop-gradients~\cite{grill2020bootstrap,chen2021exploring}) or exploration of training tricks (\eg, multi-crop~\cite{caron2020unsupervised} and strong data augmentation~\cite{tian2020makes}). Especially the three variants of UniCLR can naturally avoid the collapse solution with only the proposed loss designs. On the other hand, as a flexible framework, the existing architecture designs or training tricks can be easily integrated into UniCLR, which will be demonstrated in the experimental part.

\section{Experiments}
In this section, we conduct experiments including extensive ablation studies on multiple benchmark datasets to compare UniCLR, especially SimAffinity, SimWhiteining and SimTrace, with the state-of-the-art methods.

\subsection{Implementation}
\textbf{Datasets.}
Four commonly used benchmark datasets, including CIFAR-10~\cite{krizhevsky2009learning}, CIFAR-100~\cite{krizhevsky2009learning}, ImageNet-100~\cite{tian2019contrastive} and ImageNet-1K~\cite{deng2009imagenet} are adopted in our paper, whose details are shown in the Appendix~\ref{appendix_A}.

\textbf{Image augmentation.}
Following the literature~\cite{chen2020simple,chen2021exploring}, we apply the common data augmentations from PyTorch for all models, including (a) RandomResizedCrop with the scale in the range of $[0.2, 1.0]$, (b) Randomly applying ColorJitter with a probability of $0.8$, (c) RandomGrayscale with a probability of $0.2$, (d) Randomly applying GaussianBlur with the sigma in the scope of $[0.1, 2.0]$ and the probability of $0.5$, (e) RandomHorizontalFlip with a probability of $0.5$.

\textbf{Architecture.}
For all models, the encoder consists of a ResNet~\cite{he2016deep} backbone followed by a MLP-based projector network. Note that $f_{\theta}$ and $f_{\theta_{ema}}$ shown in Fig.~\ref{Framwork} have the same architecture except for using different parameters update strategies. That is to say, $f_{\theta_{ema}}$ uses exponential moving average as well as the stop-gradient operation to update the parameters. Specifically, we utilize a ResNet-50 backbone with 2048 output units as well as a projector network of three linear layers (each of which has 2048 output units) as our encoder for ImageNet-1K and ImageNet-100. For CIFAR-10 and CIFAR-100, the backbone is a ResNet-18 (with 512 output units) and the projector network also has three layers, but with 2048, 2048, 256 output units, respectively. Note that all the first two layers in the projector networks are followed by a batch normalization (BN) layer and a rectified linear unit (ReLU) layer, while only the BN layer is used in the last MLP layer.

\textbf{Optimization.}
Following the literature, we use a LARS optimizer with the momentum of $0.9$ and weight decay of $1e-6$, and use a cosine learning rate scheduler with a 10-epoch warmup on ImageNet-1K. Typically, the learning rate is set to $lr=\frac{batch\_size}{256} \times base\_lr$ and $base\_lr$ is set to $0.05$ for ImageNet-1K and ImageNet-100 while $0.03$ for CIFAR-10 and CIFAR-100. If not specified, all the experiments are trained $200$ epochs with the batch size of $512$ on ImageNet-1K. For the experiments on the small datasets, our all models are trained $200$ epochs on ImageNet-100 and $1000$ epochs on CIFAR-10 and CIFAR-100. The hyper-parameter $\gamma$ is set to $0.01$ by default.

\textbf{Evaluation protocol.}
Following the literature~\cite{zhang2016colorful,oord2018representation,chen2020simple,grill2020bootstrap,chen2021exploring}, we adopt a linear evaluation protocol to evaluate the features learned by different variants of UniCLR on all datasets. To be specific, given a ResNet backbone after unsupervised pre-training, a random initialized linear classifier as well as a global average pooling layer is further trained on top of the frozen backbone. More specifically, for datasets of CIFAR-10, CIFAR-100 and ImageNet-100, the linear classifier is trained for $100$ epochs using a SGD optimizer with a cosine decay schedule, where the base learning rate is $30$, weight decay is $0$, momentum is $0.9$ and batch size is $256$. As for ImageNet-1K, the linear classifier is learned with a LARS optimizer~\cite{LARS} and a cosine decay schedule for $90$ epochs, where the base learning rate is set to $0.02$, weight decay is $0$, momentum is $0.9$, and batch size is of $1024$. After training, we perform the evaluation on the center cropped images on the evaluation or test set. 

\begin{table*}[!tp] \small
\extrarowheight=-2pt
\tabcolsep=4pt
\centering
\caption{\textbf{Ablation study of UniCLR on ImageNet-1K}. All unsupervised methods are trained with a  \textit{ResNet-50} backbone and reported with the top-$1$ accuracy (\%).}
\vspace{-2mm}
\begin{tabular}{l c c c c c | c c c c}
\toprule[1pt]
 \multirow{2}{*}{Method}  &\multirow{2}{*}{Epochs}  &\multirow{2}{*}{Batch Size}  &\multicolumn{3}{c|}{\textit{Components in UniCLR}} &\multicolumn{3}{c}{\textit{More training tricks}}  &\multirow{2}{*}{Top-1}\\
\cmidrule{4-6} \cmidrule{7-9}
                          &                         & & Temperature  &Symmetric Loss   & Whitening  &Momentum  &Multi-crop  &Predictor   &  \\ 
\midrule
\multirow{7}{*}{\textbf{SimWhitening}}   &$200$  & $512$   &\XSolidBrush  &\XSolidBrush  &\Checkmark  &\XSolidBrush  &\XSolidBrush &\XSolidBrush   &58.01 \\ 
                                         &$200$  & $512$   &\Checkmark    &\XSolidBrush  &\Checkmark  &\XSolidBrush  &\XSolidBrush &\XSolidBrush   &66.64 \\ 
                                         &$200$  & $512$   &\XSolidBrush  &\Checkmark    &\Checkmark  &\XSolidBrush  &\XSolidBrush &\XSolidBrush   &\textbf{68.88} \\ 
                                         &$200$  & $512$   &\Checkmark    &\Checkmark    &\Checkmark  &\XSolidBrush  &\XSolidBrush &\XSolidBrush   &68.78 \\ 
                                        \cdashline{4-10}[1pt/2pt] 
                                         &$200$  & $512$   &\Checkmark    &\Checkmark    &\Checkmark  &\Checkmark    &\XSolidBrush &\XSolidBrush   &69.88 \\ 
                                         &$200$  & $512$   &\Checkmark    &\Checkmark    &\Checkmark  &\Checkmark    &\Checkmark   &\XSolidBrush   &70.74 \\ 
                                         &$200$  & $512$   &\Checkmark    &\Checkmark    &\Checkmark  &\Checkmark    &\Checkmark   &\Checkmark     &\textbf{71.67} \\ 
\midrule
\multirow{7}{*}{\textbf{SimAffinity}}    &$200$  & $512$   &\XSolidBrush  &\XSolidBrush  &\XSolidBrush &\XSolidBrush  &\XSolidBrush &\XSolidBrush  &64.46 \\ 
                                         &$200$  & $512$   &\Checkmark    &\XSolidBrush  &\XSolidBrush &\XSolidBrush  &\XSolidBrush &\XSolidBrush  &68.43 \\ 
                                         &$200$  & $512$   &\XSolidBrush  &\Checkmark    &\XSolidBrush &\XSolidBrush  &\XSolidBrush &\XSolidBrush  &69.73 \\ 
                                         &$200$  & $512$   &\Checkmark    &\Checkmark    &\XSolidBrush &\XSolidBrush  &\XSolidBrush &\XSolidBrush  &\textbf{69.77} \\ 
                                         \cdashline{4-10}[1pt/2pt] 
                                         &$200$  & $512$   &\Checkmark    &\Checkmark    &\XSolidBrush &\Checkmark    &\XSolidBrush &\XSolidBrush  &70.32 \\ 
                                         &$200$  & $512$   &\Checkmark    &\Checkmark    &\XSolidBrush &\Checkmark    &\Checkmark   &\XSolidBrush  &72.93 \\ 
                                         &$200$  & $512$   &\Checkmark    &\Checkmark    &\XSolidBrush &\Checkmark    &\Checkmark   &\Checkmark  &\textbf{73.84} \\ 
\bottomrule
\end{tabular} 
\vspace{-3mm}
\label{table:ablation_study_imagenet1k}
\end{table*}

\subsection{Performance on ImageNet-1K}
\label{section_experiments}
\subsubsection{Ablation Study of UniCLR on ImageNet-1K}
To systematacially investigate the effects of the major components in the UniCLR framework, we conduct an ablation study on ImageNet-1K. Specifically, two variants of UniCLR are implemented, \emph{i.e.,} SimWhitening and SimAffinity, and three key components are specially analyzed, including the \textit{temperature $\tau$}, the \textit{symmetric loss} and the \textit{whitening operation}. Note that we do not include SimTrace, because using temperature or symmetric loss is meaningless for SimTrace due to the property of trace operation, while using whitening is crucial for SimTrace. In addition, three general tricks, including momentum encoder, multi-crop and asymmetric predictor, are also evaluated. From the results in Table~\ref{table:ablation_study_imagenet1k}, we have the following observations:
\begin{itemize}
    \item \textbf{Temperature can notably boost the performance.} From the results, we can see that the simple use of temperature, \emph{i.e.,} $\tau=0.5$, can notably boost the performance of SimWhitening and SimAffinity from $58.01\%$ to $66.64\%$ ($8.63\%$ absolute improvements) and from $64.46\%$ to $68.43\%$ ($3.97\%$ absolute improvements), respectively. This is because a small temperature can reweight the importance of different negatives, especially hard negatives, which is consistent with the analysis in SimCLR~\cite{chen2020simple}. Although using temperature has been a default operation in SSL, we believe that its huge effect and potential are somewhat overlooked.
    \item \textbf{Symmetric loss can notably boost the performance.} As shown in Table~\ref{table:ablation_study_imagenet1k}, the symmetric loss can significantly improve the performance of SimWhitening and SimAffinity from $58.01\%$ to $68.88\%$ ($10.87\%$ absolute improvements) and from $64.46\%$ to $69.73\%$ ($5.27\%$ absolute improvements), respectively. This strongly demonstrates the effectiveness of the proposed symmetric loss. The reason is that, as a regularization method, the symmetric loss can effectively alleviate the false negative problem, thus benefits the training.
    \item \textbf{Whitening will damages the performance of contrastive methods, when negatives are involved.} As seen, in most cases, SimWhitening (using whitening) performs worse than SimAffinity (not using whitening). The reason may be that the decorrelation of the whitening operation will exacerbate the problem of false negatives. In other words, the decorrelation will draw the positive pair and false negatives closer to some extent, however, the contrasting with the false negatives in contrastive methods will make unsupervised models confused. Fortunately, we find that the performance loss caused by whitening can be recovered by using either temperature or the proposed symmetric loss, which verifies the above guess.
     \item \textbf{More training tricks can further enhance the performance.} The effect of multi-crop has been demonstrated by many recent SSL works~\cite{caron2020unsupervised,W_MSE_ICML2021}. Therefore, no surprise, multi-crop can consistently improve the performance of both SimWhitening and SimAffinity. A new finding is that using momentum encoder can effectively improve the performance. Another finding is that, asymmetric predictor can be used as a training trick for contrastive-learning based SSL methods to further significantly improve the final performance.
     \item \textbf{The simple SimAffinity can achieve the best result.} As seen, using the common temperature and a simple symmetric loss, SimAffinity, a simple enough baseline, can achieve a remarkable accuracy of $69.77\%$ on ImageNet-1K. With using some generic training tricks, SimAffinity could obtain the best result in Table~\ref{table:ablation_study_imagenet1k}. This further demonstrates the excellent potential and strong scalability of the proposed UniCLR framework.
\end{itemize}

\begin{table*}[!tp] \small
\extrarowheight=-2pt
\tabcolsep=5pt
\centering
\caption{\textbf{Linear evaluation on ImageNet-1K}. All methods are trained with \textit{ResNet-50} and top-$1$ accuracy (\%) is reported. \textit{Cont.}, \textit{Non-Cont.}, \textit{Consist.}, and \textit{Whit.} indicate contrastive, non-contrastive, consistency regularization and whitening based SSL methods, respectively.}
\begin{tabular}{l c c c c  c c c}
\toprule[1pt]
 Method  &Type & Epochs & Batch Size  & Negatives &Asymmetry  & Whitening   & Top-1 \\ 
\midrule
NPID~\cite{wu2018unsupervised}   &Cont.     &$200$  & $256$   &\Checkmark    &\XSolidBrush  &\XSolidBrush  & 54.0 \\
LA~\cite{zhuang2019local}        &Cont.     &$200$  & $128$   &\Checkmark    &\XSolidBrush  &\XSolidBrush  & 60.2 \\
MoCo~\cite{he2020momentum}       &Cont.     &$200$  & $256$   &\Checkmark    &\Checkmark    &\XSolidBrush  & 60.6 \\
CPC v2~\cite{henaff2020data}     &Cont.     &$200$  & $512$   &\Checkmark    &\XSolidBrush  &\XSolidBrush  & 63.8 \\
SimCLR~\cite{chen2020simple}     &Cont.     &$200$  & $256$   &\Checkmark    &\XSolidBrush  &\XSolidBrush  & 64.3 \\
SimCLR~\cite{chen2020simple}     &Cont.     &$200$  & $4096$  &\Checkmark    &\XSolidBrush  &\XSolidBrush  & 66.8\\
MoCo v2~\cite{chen2020improved}  &Cont.     &$200$  & $256$   &\Checkmark    &\Checkmark    &\XSolidBrush  & 67.5 \\
CO2~\cite{CO2_ICLR2021}          &Consist.  &$200$  & $256$   &\Checkmark    &\XSolidBrush  &\XSolidBrush  & 68.0 \\
SimSiam~\cite{chen2021exploring} &Non-Cont. &$200$  & $256$   &\XSolidBrush  &\Checkmark    &\XSolidBrush  & 70.0 \\
BYOL~\cite{grill2020bootstrap}   &Non-Cont. &$200$  & $4096$  &\XSolidBrush  &\Checkmark    &\XSolidBrush  & 70.6 \\
\\[-1em]
\midrule 
\\[-1em]
\textbf{SimWhitening} (Ours)     &Cont. \& Whit.  &$200$  & $512$   &\Checkmark    &\Checkmark  &\Checkmark    & \textbf{71.6} \\ 
\textbf{SimTrace} (Ours)         &Non-Cont.       &$200$  & $512$   &\XSolidBrush  &\XSolidBrush  &\Checkmark    & \textbf{69.1} \\ 
\textbf{SimTrace} (Ours)         &Non-Cont.       &$200$  & $512$   &\XSolidBrush  &\Checkmark  &\Checkmark    & \textbf{71.7} \\ 
\textbf{SimAffinity} (Ours)      &Cont. \& Consist.  &$200$  & $512$   &\Checkmark    &\Checkmark  &\XSolidBrush  & \textbf{73.8} \\
\\[-1em]
\midrule 
\\[-1em]
\\[-1em]
SwAV~\cite{caron2020unsupervised}        &Non-Cont. &$400$    & $4096$  &\XSolidBrush  &\Checkmark    &\XSolidBrush  & 70.1 \\
W-MSE~\cite{W_MSE_ICML2021}              &Whit.     & $400$   & $1024$  &\XSolidBrush  &\XSolidBrush  &\Checkmark    & 72.5 \\
Zero-CL~\cite{Zero-CL}                   &Whit.     & $400$   & $1024$  &\XSolidBrush  &\XSolidBrush  &\Checkmark    & 72.6 \\
Barlow Twins~\cite{zbontar2021barlow}    &Whit.     & $1000$  & $2048$  &\XSolidBrush  &\XSolidBrush  &\Checkmark    & 73.2 \\
VICReg~\cite{VICReg_ICLR2022}            &Whit.     & $1000$  & $2048$  &\XSolidBrush  &\XSolidBrush  &\Checkmark    & 73.2 \\
\bottomrule
\end{tabular} 
\vspace{-3mm}
\label{table:sota_results_cmp_imagenet1k}
\end{table*}

\subsubsection{Comparison with the State of the Art}
We evaluate the UniCLR framework by implementing three variants with it, \emph{i.e.,} SimTrace, SimWhitening and SimAffinity, on ImageNet-1K. Specifically, the symmetric loss is used for both SimAffinity and SimWhitening, and the temperature $\tau\!=\!0.5$ is only used for SimAffinity. As for SimTrace, only Eq.~(\ref{fun3}) is employed without either the symmetric loss or temperature. According to the above ablation study, we also use momentum encoder, multi-crop and asymmetric predictor for all variants. All of our models are trained with a batch size of $512$ for $200$ epochs.

The linear evaluation results are reported in Table~\ref{table:sota_results_cmp_imagenet1k}. From the results, we can observe that all the proposed three variant methods can achieve notable performance on ImageNet-1K. To be specific, with the 200-epoch training, SimWhitening, SimTrace and SimAffinity all can obtain much better results than the existing contrastive, non-contrastive and consistency regularization based methods. Notably, SimAffinity obtains an accuracy of $73.8\%$, which gains $3.8\%$ and $3.2\%$ improvements over SimSiam and BYOL under the same 200-epoch training, respectively. Surprisingly, this result is tangibly better than Barlow Twins and VICReg that adopt a 1000-epoch training process. It is also worth mentioning that when no asymmetric designs are used, SimTrace obtains a remarkable accuracy of $69.1\%$, which confidently demonstrates that SimTrace can naturally avoid the mode collapse problem. Also, when the asymmetric designs are further used as training tricks, the performance of SimTrace is significantly improved from $69.1\%$ to $71.7\%$ ($2.6\%$ absolute improvements). We believe the promising potential of the proposed UniCLR has already been verified in the above experiment. On the other hand, as a unified framework, UniCLR also has the promising scalability for the future research.

\begin{figure*}[!tp]
\centering
\includegraphics[width=0.66\textwidth]{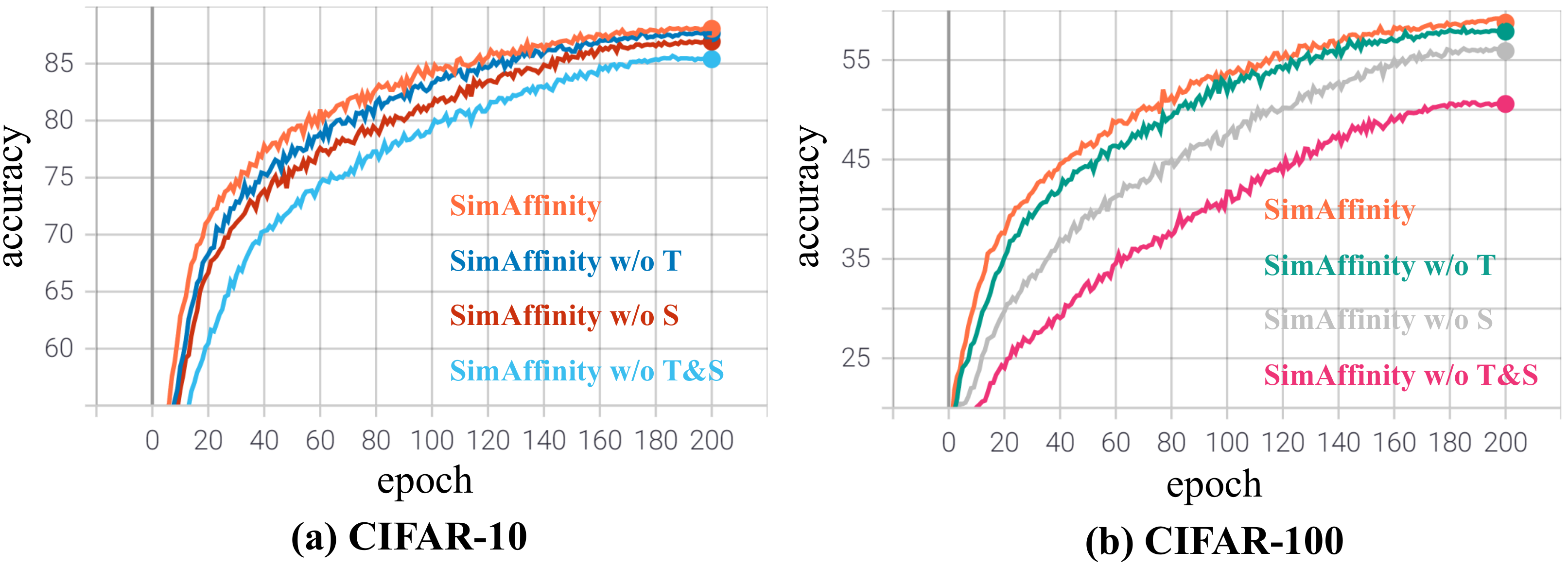}
\vspace{-2mm}
\caption{\textbf{$k$-NN evaluation on CIFAR-10 and CIFAR-100.} (a) Experimental results on CIFAR-10. (b) Experimental results on CIFAR-100. SimAffinity w/o T, w/o S, and w/o T\&S indicate SimAffinity without temperature, without symmetric loss, and without both parts, respectively. All models are trained for $200$ epochs, and different colors denote different methods.}
\vspace{-3mm}
\label{KNN}
\end{figure*}

\subsection{Performance on Small Data}
\subsubsection{Comparison with the State of the Art}
We also perform the three variants of UniCLR, \emph{i.e.,} SimAffinity, SimWhitening and SimTrace, on three small-scale benchmark datasets, \emph{i.e.,} CIFAR-10~\cite{krizhevsky2009learning}, CIFAR-100~\cite{krizhevsky2009learning}, and ImageNet-100~\cite{tian2019contrastive}. Following the literature~\cite{tian2019contrastive,chen2020simple,chen2021exploring}, we adopt ResNet-18 as the backbone and train our models for $1000$ epochs on both CIFAR-10 and CIFAR-100. As for ImageNet-100, we employ ResNet-50 as the backbone and only train our models for $200$ epochs. Specifically, four representative methods, including SimCLR~\cite{chen2020simple}, SimSiam~\cite{chen2021exploring}, DINO~\cite{caron2021emerging} and Zero-CL~\cite{Zero-CL}, are taken as competitors.

The results are reported in Table~\ref{table:small_datasets}. As seen, on both CIFAR-10 and ImageNet-100, all the three variants of UniCLR could achieve remarkable results, and SimAffinity consistently performs the best on both of two datasets. Specifically, SimAffinity achieves a new state-of-the-art result of $80.12\%$ on ImageNet-100 with 200-epoch training, which is higher than SimCLR, SimSiam, DINO and Zero-CL with 400-epoch training by $2.64\%$, $3.08\%$, $5.28\%$ and $0.86\%$, respectively. On CIFAR-100, both SimWhitening and SimAffinity still can obtain competitive results on par with the competitors except Zero-CL. Overall, the effectiveness of the proposed UniCLR framework on small-scale data also has been demonstrated.

\begin{table}[!tp] \small
\extrarowheight=-2pt
\tabcolsep=3pt
\centering
\caption{\textbf{Linear evaluation on three small datasets.} The top-$1$ accuracy (\%) is reported. For each dataset, the best and second best methods are highlighted.}
\vspace{-2mm}
\begin{tabular}{l l  c c c}
\toprule[1pt]
 Dataset  & Method  & Backbone & Epochs & Top-1 \\
\midrule
\multirow{7}{*}{CIFAR-10}   &SimCLR~\cite{chen2020simple}          & ResNet-18 & 1000 & 90.74 \\
                            &SimSiam~\cite{chen2021exploring}      & ResNet-18 & 1000 & 90.51 \\
                            &DINO~\cite{caron2021emerging}         & ResNet-18 & 1000 & 89.19 \\
                            &Zero-CL~\cite{Zero-CL}                & ResNet-18 & 1000 & 90.81 \\
                            \cmidrule{2-5}
                            &\textbf{SimTrace}               & ResNet-18 & 1000 & 90.07 \\
                            &\textbf{SimWhitening}           & ResNet-18 & 1000 & \textbf{90.91} \\
                            &\textbf{SimAffinity}            & ResNet-18 & 1000 & \textbf{91.12} \\
                            
\midrule
\multirow{7}{*}{CIFAR-100}  &SimCLR~\cite{chen2020simple}          & ResNet-18 & 1000 & 65.39 \\
                            &SimSiam~\cite{chen2021exploring}      & ResNet-18 & 1000 & 65.86 \\
                            &DINO~\cite{caron2021emerging}         & ResNet-18 & 1000 & 66.38 \\
                            &Zero-CL~\cite{Zero-CL}                & ResNet-18 & 1000 & \textbf{70.33} \\
                            \cmidrule{2-5}
                            &\textbf{SimTrace}               & ResNet-18 & 1000 & 63.09 \\
                            &\textbf{SimWhitening}           & ResNet-18 & 1000 & 66.12 \\
                            &\textbf{SimAffinity}            & ResNet-18 & 1000 & \textbf{66.76} \\
\midrule
\multirow{7}{*}{ImageNet-100}   &SimCLR~\cite{chen2020simple}      & ResNet-50 & 400 & 77.48 \\
                                &SimSiam~\cite{chen2021exploring}  & ResNet-50 & 400 & 77.04 \\
                                &DINO~\cite{caron2021emerging}     & ResNet-50 & 400 & 74.84 \\
                                &Zero-CL~\cite{Zero-CL}            & ResNet-50 & 400 & 79.26 \\
                                \cmidrule{2-5} 
                                &\textbf{SimTrace}           & ResNet-50 & 200 & 79.49 \\
                                &\textbf{SimWhitening}       & ResNet-50 & 200 & \textbf{79.83} \\
                                &\textbf{SimAffinity}        & ResNet-50 & 200 & \textbf{80.12} \\
\bottomrule
\end{tabular} 
\vspace{-3mm}
\label{table:small_datasets}
\end{table}

\subsubsection{Ablation Study of SimAffinity on Small Data}
We also conduct an ablation study on two small datasets, \emph{i.e.,} CIFAR-10 and CIFAR-100, to further investigate the effectiveness of the temperature and the proposed symmetric loss. To be specific, SimAffinity and its three variants, \emph{i.e.,} \textit{SimAffinity without temperature (w/o T for short)}, \textit{SimAffinity without symmetric loss (w/o S for short)} and \textit{SimAffinity without temperature and symmetric loss (w/o T\&S for short)}, are trained for $200$ epochs on both datasets. After that, the $k$-NN classification accuracy curves during pre-training are illustrated in Fig.~\ref{KNN}.

From the curves, we are able to observe that (1) either temperature alone or symmetric loss alone can significantly accelerates the convergence of training of the unsupervised models; (2) the symmetric loss makes the convergence faster than using temperature, which is consistent with the finding in the ablation study on ImageNet-1K with a linear evaluation protocol; (3) similar to the observation in Section~\ref{section_experiments}, under the $k$-NN evaluation protocol, integrating the symmetric loss with the temperature together can clearly produce further improvements over each individual part alone.

\section{Conclusion}
In this paper, we present a simple unified contrastive learning representation framework \textit{UniCLR} and its three variants \textit{SimAffinity}, \textit{SimWhitening} and \textit{SimTrace}, from a new perspective of affinity matrix. In addition, a simple symmetric loss for alleviating the false negative problem is also proposed. From the extensive analyses and experiments, we draw the following conclusions: (1) From the affinity matrix perspective, many existing SSL methods can be explained and unified with the UniCLR framework, and the simple UniCLR framework achieves superior results than the state of the art on multiple benchmarks; (2) The effect of temperature is confirmed in that it can notably accelerate the convergence of model's training process; (3) The proposed symmetric loss can also  clearly accelerate the convergence of models and significantly boost the performance,  showing how to effectively alleviate the false negative problem is a promising direction in SSL; (4) Whitening is a key to help non-contrastive SSL methods avoid trivial solutions but is harmful to contrastive SSL. Integrating the proposed symmetric loss with whitening  can significantly alleviate this problem; and (5) The proposed UniCLR provides excellent scalability. Future work will explore how to unify other kinds SSL methods, such as clustering based methods, into this framework.

{\small
\bibliographystyle{ieee_fullname}
\bibliography{egbib}
}

\appendix

\section{Datasets}
\label{appendix_A}
There are four benchmark datasets used in this work, including ImageNet-1K~\cite{deng2009imagenet}, ImageNet-100~\cite{tian2019contrastive}, CIFAR-10~\cite{krizhevsky2009learning}, and CIFAR-100~\cite{krizhevsky2009learning}. The details of these dataset are as follows:
\begin{itemize}
    \item \textbf{ImageNet-1K} is a widely-used large-scale classification dataset, which spans $1000$ classes with $1.2$ million training images and $50000$ validation images. We use the $1.2$ million images without labels for unsupervised pre-training and the $50000$ images for the final test. The image size is set to $224\times 224$.
    \item \textbf{ImageNet-100} is a highly-used subset of ImageNet-1K~\cite{deng2009imagenet} for unsupervised representation learning. It is composed of $100$ classes randomly selected from ImageNet-1K and has a total number of $0.13$ million images. Also, the image size is set to $224\times 224$.
    \item \textbf{CIFAR-10} is a small-scale dataset and contains a total number of $60000$ $32\times32$ images in $10$ classes, where there are $6000$ images in each class. Specifically, we take $50000$ and $10000$ images for training and test, respectively.
    \item \textbf{CIFAR-100} is also a small dataset, which consists of $100$ classes with $600$ $32\times32$ images in each class. Following the literature, we take $50000$ images ignoring the labels for training, and take the remaining images for test.
\end{itemize}

\section{Proof of Proposition ~\ref{theorem1}}
\label{appendix_B}
We present the following Proposition~\ref{theorem1} and proof to further explain Eq.~(\ref{fun4}) in the main paper, showing that the affinity matrix and covariance matrix can be converted into each other. In this sense, the affinity matrix and covariance matrix can be essentially optimized with the same single trace loss function of SimTrace.  

\begin{proposition}\label{theorem1}
Suppose $Z\in\mathbb{R}^{D\times N}$, $Z'\in\mathbb{R}^{D\times N}$, $\hat{Z}=\{Z, Z'\}\in\mathbb{R}^{D\times 2N}$, $\Sigma$ is the covariance matrix calculated from $\hat{Z}$, which is positive semi-definite, and $\Sigma^{-\frac{1}{2}}$ is the square root of the inverse of $\Sigma$, it can be obtained that 
\begin{equation}\label{FunT1} \small
\begin{split}
 \text{Trace}\big(Z^\top\Sigma^{-1}Z') = \text{Trace}\big(\Sigma^{-1}Z'Z^\top) \,.
\end{split}
\end{equation}
\end{proposition}

\begin{proof}
Specifically, according to the basic properties of matrix and $\Sigma^{-1}=\Sigma^{-\frac{1}{2}}\Sigma^{-\frac{1}{2}}$, we have
\begin{equation}\label{funT2}\small
\begin{split}
\text{Trace}\big(Z^\top\Sigma^{-1}Z')
    &=  \text{Trace}\big(Z^\top \Sigma^{-\frac{1}{2}}\Sigma^{-\frac{1}{2}}Z'\big)\\
    &\footnotesize\text{\emph{\big(Cyclic property of trace, i.e., Trace(ABCD)=Trace(BCDA)\big)}}\\
    & =  \text{Trace}\big(\Sigma^{-\frac{1}{2}}\Sigma^{-\frac{1}{2}}Z'Z^\top)\\
    &\footnotesize\text{\emph{\big(because $\Sigma^{-1}=\Sigma^{-\frac{1}{2}}\Sigma^{-\frac{1}{2}}$\big)}}\\
    &=   \text{Trace}\big(\Sigma^{-1}Z'Z^\top\big)\,.
\end{split}
\end{equation}
Therefore, we can obtain Eq.~(\ref{FunT1}). Hence proved.
\end{proof}

\begin{algorithm*}[!t]
\caption{UniCLR Pseudocode, PyTorch-like}
\label{algorithm1}
\begin{lstlisting}[language=Python]
Whitening = False  # whitening option
Trace = False      # SimTrace option
t     = 0.5        # temperature
gamma = 0.01       # hyper-parameter 
for X1, X2 in loader:
    
    Z1  = self.encoder(X1)   # D * N, N is the batch size 
    Z2  = self.encoder(X2)   # D * N, D is the feature dimension
    
    affinity = self.cal_affinity_matrix(z1, z2)       # N * N
    
    if Trace:
        loss = -torch.trace(affinity)                 # SimTrace loss
    else:
        targets = range(0, Z1.shape[0])
        loss1   = CrossEntropyLoss(affinity, targets) # CE loss 
        loss2   = torch.norm(affinity - affinity.T)   # Symmetry loss
        loss    = loss1 + gamma * loss2
        
    loss.backward()   # back propagation
    optimizer.step()  # SGD update

def cal_affinity_matrix(Z1, Z2):
    if Whitening:
        Z1, Z2 = whitening(Z1, Z2)    # using whitening operation
    Z1 = torch.norm(Z1, p=2, dim=-1)
    Z2 = torch.norm(Z2, p=2, dim=-1)
    affinity_matrix = torch.matmul(Z1, Z2.T)
    
    return affinity_matrix / t        # scaling by temperature
\end{lstlisting}
\end{algorithm*}

\section{SimAffinity Loss vs. InfoNCE Loss}
\label{appendix_C}
In this section, we present further explanation for the affinity matrix based contrastive learning (SimAffinity) loss and reveal the relationship between the proposed SimAffinity loss and InfoNCE loss. We first recall the InfoNCE loss in SimCLR as below: 
\begin{equation}
\mathcal{L}_\text{InfoNCE}=-\log\frac{\exp(\bm{z}^\top_i\Tilde{\bm{z}}_i/\tau)}{\sum_{j=1}^{2N}\mathbb{I}_{j\neq i}\exp(\bm{z}^\top_i \bm{z}_j/\tau)}\,,
\end{equation}
where $\bm{z}_i$ and $\Tilde{\bm{z}}_i$ are positive samples, $\bm{z}_j$ denotes the negative sample, $\mathbb{I}_{k\neq i}$ is an indicator function, $N$ is the number of batch size and $\tau$ is a temperature.

As for the proposed SimAffinity loss (\ie, \textit{Eq.~(\ref{fun1}) in the main paper}), we can take the first row of the affinity matrix as an example, whose cross-entropy loss can be formulated as follows:
\begin{equation}
 \mathcal{L}_\text{CE}=-\log\frac{\exp(\bm{z}^\top_i\Tilde{\bm{z}}_i/\tau)}{\sum_{j=1}^{N}\exp(\bm{z}^\top_i \bm{z}_j/\tau)}\,.
\end{equation}

According to the above two formulations, we can see that SimAffinity is essentially similar to the InfoNCE based methods. The key difference is that our SimAffinity only considers N-1 rather than 2N-2 negative samples for each positive sample. On the one hand, in terms of calculation, SimAffinity is more efficient than the InfoNCE based methods for only considering N-1 negative samples. Moreover, SimAffinity can be implemented by pure matrix manipulation and thus enjoys higher efficiency, while SimCLR can not. On the other hand, more importantly, SimAffinity can naturally enjoy some wonderful properties of the affinity matrix, such as using the proposed symmetric loss or performing the whitening operation. Also, the experimental results in Table~2 of the main paper have successfully demonstrated that SimAffinity can significantly outperform the existing InfoNCE based SSL methods.

\section{Pseudo Code for UniCLR}
\label{appendix_D}
The pseudo-code of the proposed UniCLR is shown in Algorithm~\ref{algorithm1}.

\section{Further Explanation for Symmetric Loss}
\label{appendix_E}
In the main paper, we claim that the proposed symmetric loss can alleviate the false negative problem to some extent. Therefore, in this section, we will give some intuitive examples to further explain the above claim.

Suppose we have $Z=\{z_1, \cdots, z_N\}$ and $Z'=\{z'_1, \cdots, z'_N\}$, which denote two different views (feature representations) of the same mini-batch of $N$ images. That is to say, any pair of $<z_i, z'_i>|^N_{i=1}$ is a true positive pair. Intuitively, given a certain positive pair of $<z_1, z'_1>$, to tackle the false negative problem, we can make the relationships between $z_1$ and all the other $N-1$ negative samples, \ie, $A=\{z^\top_1 z_2, \cdots, z^\top_1 z_N\}$ in one view, be similar to $B=\{z'^\top_1 z'_2, \cdots, z'^\top_1 z'_N\}$ in another view. To be specific, if there is a false negative sample $z_i\in Z$, it will naturally has a high similarity with the true positive sample $z_1$, because their real semantic labels are the same. The principle is that if we can make the similarity between $z'_1$ and $z'_i$ (another view of the false negative sample $z_i$) in another view as high as the similarity between $z_1$ and $z_i$, we can alleviate the false negative problem to some extent. Therefore, a straightforward way to alleviate the above problem is that we can design a regularization term to close such a distance between $A$ and $B$ by minimizing $D(A,B)$.

To move forward a small step, our proposed symmetric loss further calculates a kind of cross positive-negative similarities between $Z$ and $Z'$ as $A'=\{z^\top_1 z'_2, \cdots, z^\top_1 z'_N\}$ and $B'=\{z'^\top_1 z_2, \cdots, z'^\top_1 z_N\}$, and make $A'\equiv B'$ with the mean squared error (through Frobenius norm of matrix). The advantage is that we can implement such a constraint as a simple symmetric operation on the affinity matrix, benefitting from using affinity matrix as the objective. More importantly, because of using such cross positive-negative similarities, our symmetric loss could benefit from the cross diversity, learning invariant representations.

To better understand the above principle, we can present another intuitive example. As explained in the main paper, to optimize a SSL model by the proposed SimAffinity loss in Eq.(1) in the main paper, we can calculate $P=\{z^\top_1 z'_1, z^\top_1 z'_2, \cdots, z^\top_1 z'_N\}$ (\ie, the first row of the affinity matrix) and make $P$ close to a one-hot vector $Q=\{1, 0, \cdots, 0\}$ by cross-entropy loss. Clearly, because of the one-hot ground truth, the similarity between $z_1$ and $z'_1$ will be optimized to be as large as possible. In contrast, the similarity between $z_1$ and all the $N-1$ negative samples will be equally optimized to be as small as possible, so as to the false negative samples. However, using the proposed symmetric loss as a regularization term, we can maintain the relative similarities between $z_1$ and all the other negative samples as stable as possible, especially the similarities with the false negative samples.

\end{document}